\documentclass[10pt]{article}
\usepackage{geometry}
\usepackage{graphicx}
\usepackage{amsfonts, amsmath, amsthm, amstext, amssymb, amscd}
\usepackage[all]{xy}
\usepackage{color}
\usepackage{ifpdf}
\usepackage{mathrsfs}
\input epsf
\usepackage{xcolor}

\usepackage{mathabx}

\usepackage{amssymb}
\usepackage{braket,amsfonts}

\usepackage{algorithm}
\usepackage{algorithmic,color}
\usepackage{pgfplots,filecontents,tikz}
\usepackage{hyperref}
\usepackage{amsthm}
\usepackage{wrapfig} 

\newtheoremstyle{thm}{0.32em}{0.32ex}{\it}{}{\bf}{.}{0.5em}{\thmname{#1}\thmnumber{ #2}\thmnote{ ({\rm #3})}}
\theoremstyle{thm}
\newtheorem{theorem}{Theorem}[section]

\newtheorem{lemma}[theorem]{Lemma}

\newtheorem{example}[theorem]{Example}

\newtheorem{corollary}[theorem]{Corollary}

\newtheorem{proposition}[theorem]{Proposition}

\newcounter{chap}
\newcounter{sect}

\usepackage{titlesec}
\usepackage[normalem]{ulem}
\titleformat{\section}
{\normalfont\large\bfseries}{\thesection}{1em}{}
\titlespacing{\section}{0cm}{0.6cm}{0.2cm}

\makeatletter
\def\enumerate{%
 \ifnum \@enumdepth >\thr@@\@toodeep\else
   \advance\@enumdepth\@ne
   \edef\@enumctr{enum\romannumeral\the\@enumdepth}%
     \expandafter
     \list
       \csname label\@enumctr\endcsname
       {\usecounter\@enumctr\def\makelabel##1{\hss\llap{##1}}%
         \addtolength{\parsep}{1.56pt}
         \addtolength{\listparindent}{0pt} 
         \addtolength{\itemsep}{-8pt} 
         \addtolength{\topsep}{-1pt}} 
 \fi}
\makeatother

\def\NN{\mathbb{N}}

\def\RR{\mathbb{R}}

\newcommand{\supp}{\operatorname{supp}}

\renewenvironment{proof}{\noindent{\em
Proof.}\hspace{0.3em}}{\hfill\qed\vspace{0ex}}

\usepackage[T1]{fontenc}
\usepackage[utf8]{inputenc}
\usepackage{authblk}
\title{\bf \large Approximation capabilities of neural networks on unbounded domains}
\author[a]{Ming-Xi Wang \thanks{Corresponding author}}
\author[b,c]{Yang Qu}
\affil[a]{PAG Investment Solutions, Geneve, Switzerland. mingxi.waeng@gmail.com}
\affil[b]{School of Mathematics, Hunan University, China, quyang@hnu.edu.cn}
\affil[c]{Xiaoxiang Research Institute of Big Data}
\usepackage{booktabs}
\usepackage{tikz-cd}
\begin{document}
\maketitle

\begin{abstract}
In this paper, we prove that a shallow neural network with a monotone sigmoid, ReLU, ELU, Softplus, or LeakyReLU activation function can arbitrarily well approximate any $L^p(p\ge2)$ integrable functions defined on $\RR\times[0,1]^n$. We also prove that a shallow neural network with a sigmoid, ReLU, ELU, Softplus, or LeakyReLU activation function expresses no nonzero integrable function defined on the Euclidean plane. Together with a recent result that the deep ReLU network can arbitrarily well approximate any integrable function on Euclidean spaces, we provide a new perspective on the advantage of multiple hidden layers in the context of ReLU networks. Lastly, we prove that ReLU network with depth 3 is a universal approximator in $L^p(\RR^n)$.
\end{abstract}




\bigskip
\noindent

\section{Introduction}\label{intro}
The universal approximation theorem of Cybenko \cite{Cybenko} states that single hidden layer neural networks with a sigmoidal activation function can arbitrarily well approximate any continuous function with support in the unit hypercube. This famous theorem justifies the representational power of feedforward neural networks. We refer to \cite{Funahashi}, \cite{Hornik1989}, \cite{Barron}, \cite{polynomial}, \cite{Pinkus} and references there for numerous variants and generalizations of this result.

Shallow neural networks are neural networks that have only one hidden layer.  In this paper, we start with investigating the representational power of shallow neural networks in $L^p(\RR^m\times[0,1]^n)(1\mathopen\leq\mathclose p \mathopen<\mathclose\infty)$. In the case $m\mathopen=\mathclose0$, it is well-known that shallow neural networks with a nonpolynomial activation function are universal approximators in $L^p([0,1]^n)$ \cite{Pinkus}. In the case $m\mathopen=\mathclose1$, $p\mathopen\ge\mathclose2$, we prove that shallow  monotone sigmoid, ReLU, ELU, softplus, or LeakyReLU networks  are universal approximators in $L^p(\RR\times[0,1]^n)$ (Theorem \ref{theorembounded11}, Theorem \ref{theorembounded}). This is a partial generalization of the previous classical result. In the case $m\mathopen>\mathclose1$, we show that shallow sigmoid, ReLU, ELU, softplus, or LeakyReLU networks express no nonzero function in $L^p(\RR^m\times[0,1]^n)$(Corollary \ref{corollarydepth}). This follows as a corollary of a stronger result that these neural networks express no nonzero function in $L^p(\RR\times\RR^+)$(Proposition \ref{theorembounded111}, Proposition \ref{theorembounded2}). Lastly, in the context of deep networks, we prove that deep ReLU networks with depth 3 is a universal approximator in $L^p(\RR^n)$.

In the literature, there have been some results concerning neural networks' representational power on unbounded domains: Under mild conditions, radial-basis-function networks are universal approximators in $L^p(\RR^n)(p \in [1, \infty))$ \cite{Park}; Shallow sigmoid networks are universal approximators in $C(\overline{\RR^n})$, where $\overline{\RR^n}$ is the one-point compactification of $\RR^n$ \cite{Ito}, \cite{Chen1991}, \cite{Huang}; Shallow neural networks with a non-constant bounded activation are universal approximators in $L^p(\RR^n, \mu)$, where $\mu$ is a finite input space environment measure \cite{Hornik1991}; Deep ReLU networks with bounded width can arbitrarily well approximate any function in $L^1(\RR^n)$ \cite{width}; Deep ReLU networks with at most $[\log_2(n+1)]$ hidden layers can arbitrarily well approximate any function in $L^p(\RR^n)(p \in [1, \infty))$ \cite{deeprelu}. Our work is an addition to the existing literature. For example, concerning expressivity of ReLU networks, we have Table \ref{tableReLU}.
\begin{small}
\begin{table}[h]
\caption{Expressivity of ReLU networks ($n\ge2$, depth: $d_h$, width: $d_w$)}
\label{tableReLU}
\begin{center}
\begin{tabular}{lllll}
\multicolumn{1}{c}{Network architecture}  &\multicolumn{1}{c}{Function space} &\multicolumn{1}{c}{Approximator?} &\multicolumn{1}{c}{Reference}\\ 
\midrule
Arbitrary width case:\\
$d_h=2$  &   $C([0,1]^n), L^p([0,1]^n)$  & $\surd$   & \cite{5726775} \\
$d_h=2$  &   $C(\overline{\RR^n})$  & $\surd$   & \cite{Ito} \\
$d_h \mathopen\leq\mathclose \lceil\log_2(n\mathopen+\mathclose1)\rceil\mathopen+\mathclose1$  &   $L^p(\mathbb{R}^n)$  & $\surd$   & \cite{deeprelu} \\
$d_h=2$  &   $L^p(\mathbb{R}\times [0,1]^{n-1})(p\mathopen\ge\mathclose2)$  & $\surd$   & This paper \\
$d_h=2$  &   $L^p(\mathbb{R}^n), L^p(\mathbb{R}^2\times [0,1]^{n-2})$  & $\times$   & This paper \\
$d_h=3 $  &   $L^p(\mathbb{R}^n)$  & $\surd$   & This paper \\
Arbitrary depth case:\\
$d_w\leq n+4$  &   $L^1(\mathbb{R}^n)$  & $\surd$   & \cite{width} \\
$d_w\leq n+1$  &   $C([0,1]^n)$  & $\surd$   & \cite{Hanin2017ApproximatingCF} \\
$d_w\leq n+2$  &   $L^p(\mathbb{R}^n)$  & $\surd$   & \cite{pmlr-v125-kidger20a}\\
$d_w\leq n$  &   $L^1(\mathbb{R}^n)$  & $\times$   & \cite{width} \\
$d_w\leq n$  &   $C([0,1]^n)$  & $\times$   & \cite{Hanin2017ApproximatingCF} \\
\end{tabular}
\end{center}
\end{table}
\end{small}


In dimension two, our results on shallow networks are summarized in Figure \ref{fig1}: (a) If $\Omega$ is
\begin{wrapfigure}{r}{0.5\textwidth}\label{fig1}
\centering
\includegraphics[width=0.5\textwidth]{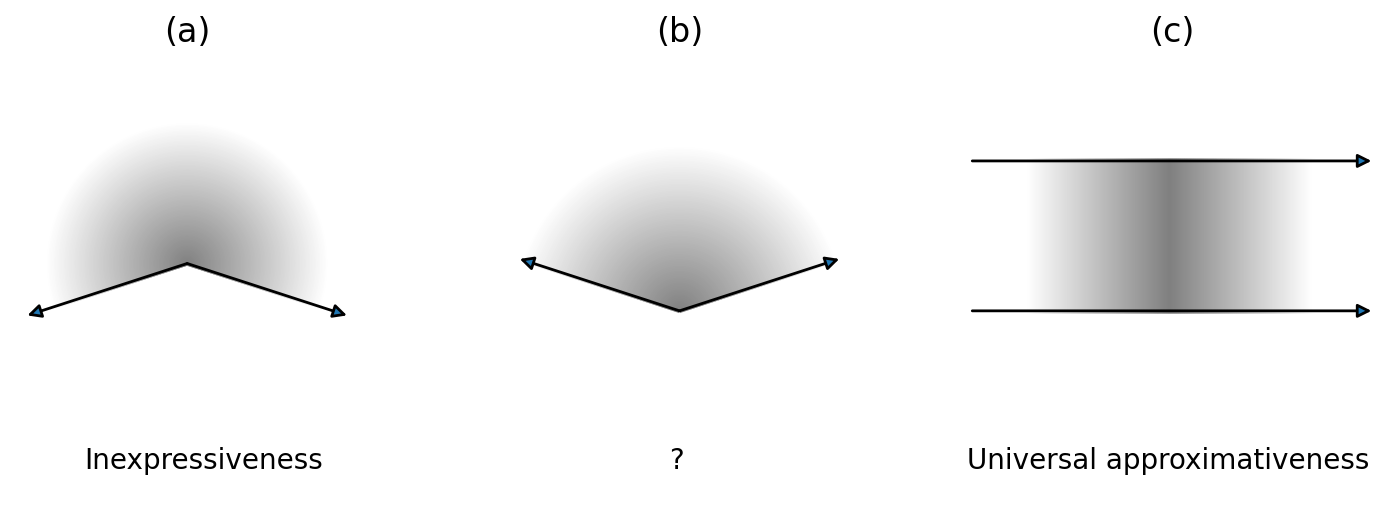}
\caption{Expressivity of shallow networks on two dimensional cones}
\label{summary}
\end{wrapfigure}
a cone formed  by a straight or reflex angle, then shallow sigmoid, ReLU, ELU, softplus, or LeakyReLU networks express no nonzero integrable functions defined on $\Omega$; (b) If $\Omega$ is a cone formed by an acute, right or obtuse angle, the representational power of neural networks on $\Omega$ is unknown; (c) If $\Omega$ is a domain bounded by two parallel lines, then shallow monotone sigmoid, ReLU, ELU, softplus, or LeakyReLU networks can arbitrarily well approximate any function in $L^p(\Omega)(p \in [2, \infty))$. A domain of type (b) is called a two-dimensional light cone and is the mathematical model of the spacetime. The representational power of neural networks on general light cones seems an interesting open problem because they are natural spaces for machine learning applications \cite{Nickel2018LearningCH},\cite{spacetime}. 



Notations: $\RR, \NN$ and $\RR^{+}$ are real numbers, natural numbers and positive real numbers; $\RR^n$ is the $n$-dimensional Euclidean space; $\langle , \rangle$ is the dot product or the target domain; $\tau_{\varrho}$ and $L_{y}$ are defined in (\ref{528}); $\mathcal{S}_{n}(\phi)$ is the space of functions defined on $\RR^n$ that can be expressed by a shallow $\phi$-network; $\widehat{h}$ is the Fourier transform of $h$; $L^p(X)$ is the space of $L^p$-integrable functions on $X$ and $||f||_{L^p}$ is the $L^p$ norm of $f$; $\phi_{\mathcal{D}}$ is the characteristic function of $\mathcal{D}$; $\lambda_n$ is the Lebesgue measue on $\RR^n$; $\Delta^n$ are difference operators.

\section{Neural networks}
\paragraph{Activation functions}
We introduce some activation functions, which are popular choices in practice. A sigmoid or sigmoidal function such as Logistic function is a real-valued function $\phi$ that satisfies
\begin{align*}
\phi(t) &\to \left\{
\begin{array}{rcl}
1 & ~ as&{t\to +\infty,}\\
0 & ~ as&{t\to -\infty}. \\
\end{array}\right.
\end{align*}
The following unbounded functions are commonly used recently \cite{relu}, \cite{elu}:
\begin{align*}
\textup{ReLU}(x) &= \left\{
\begin{array}{rcl}
x & ~ if&{x>0,}\\
0 & ~ if&{x\leq 0};
\end{array}\right. \\
\textup{ELU}(x) & = \left\{
\begin{array}{ccl}
x & ~ if&{x>0,}\\
\alpha(e^{x}-1)  & ~ if&{x\leq 0};
\end{array}\right. \\
\textup{Softplus}(x) &= \log(e^x+1); \\
\textup{LeakyReLU}(x) &= \left\{
\begin{array}{ccl}
x & ~ if&{x>0,}\\
\alpha x  & ~ if&{x\leq 0, \ \ \alpha \neq 1.}
\end{array}\right.
\end{align*}
It is useful to define the following notions. Let $\mathcal{P}$ be some property such as bounded, differentiable, monotone, or Lipschitz. A function $\phi: U \to \RR$ is called essentially $\mathcal{P}$ if there exists a measure zero set $Z$ such that $\phi|_{U \setminus Z}$ is $\mathcal{P}$; eventually $\mathcal{P}$ if there exists $X>0$ such that $\phi|_{(X, \infty)}$ and $\phi|_{(-\infty, -X)}$ are $\mathcal{P}$; eventually essentially $\mathcal{P}$ if there exists $X>0$ such that $\phi|_{(X, \infty)}$ and $\phi|_{(-\infty, -X)}$ are essentially $\mathcal{P}$. We write $\lim_{t \to x}\phi(t) = c, a.e.$, if there exists a measure zero set $Z$ such that  $\lim_{t \notin Z, t \to x}\phi(t) = c$.
\begin{example}
\textup{ReLU} and \textup{Softplus} are monotone. \textup{ELU} and \textup{LeakyReLU} are eventually monotone. The first derivative of \textup{Softplus} is monotone. The first derivatives (exist as measurable functions) of  \textup{ReLU}, \textup{ELU}, and \textup{LeakyReLU} are eventually monotone.
\end{example}
\paragraph{Representation spaces of neural networks}
Given an activation function $\phi$, a point $y\in\RR^n$, a real number $\varrho$, we define functions $\phi^{\tau_{\varrho}}, \phi^{L_{y}}, \phi^{\tau_{\varrho} L_{y}}$ by
\begin{align}\label{528}
\phi^{\tau_{\varrho}}:& x \in \RR \mapsto \phi(x+\varrho), \nonumber \\
\phi^{L_{y}}:& x \in \RR^{n} \mapsto \phi(\langle y, x\rangle),  \\
\phi^{\tau_{\varrho} L_{y}} &= (\phi^{\tau_{\varrho}})^{L_{y}}. \nonumber
\end{align}
Because $(\phi^{\tau_{\varrho}})^{L_{y}}(x) \mathopen=\mathclose\phi^{\tau_{\varrho}}(\langle y, x\rangle) \mathopen=\mathclose \phi(\langle y, x\rangle\mathopen+\mathclose\varrho)$, $\phi^{\tau_{\varrho} L_{y}}$ is simply the classical neuron model with activation function $\phi$ and parameters $\{y, \varrho\}$. A single hidden layer neural network with activation function $\phi$ (called a shallow $\phi$ network) attempts to approximate a signal $\mathcal{F}$ by $F(x)=\sum_{i=1}^{k}t_i\phi^{\tau_{\varrho_i} L_{y_i}}(x)$, where $k$ is a hyperparameter, and $y_i,\varrho_i,t_i$ are parameters subject to training. A function $F$ defined on $\RR^n$ is represented (or expressed) by a shallow $\phi$ network if there exist $k \mathopen\in\mathclose \NN, y_i\mathopen \in\mathclose \RR^{n}, t_i\mathopen \in\mathclose \RR, \varrho_i\mathopen \in\mathclose \RR$ such that $F = \sum_{i=1}^{k}t_i\phi^{\tau_{\varrho_i} L_{y_i}}.$ The space of all these functions is denoted by $\mathcal{S}_{n}(\phi)$.
\paragraph{Universal approximator}
The closure of a subset $\mathcal{S} \subset X$ of a topological space $X$ is denoted by $\overline{\mathcal{S}}$. Let $X$ be a function space on $\Omega \subset \RR^n$ with some topology. The shallow $\phi$ network (or $\mathcal{S}_{n}(\phi)$) is called a universal approximator in $X$ if $\overline{\mathcal{S}_{n}(\phi) \cap X}=X.$

\section{Approximation capabilities}\label{section3}
In this section, we show that shallow monotone sigmoid, ReLU, ELU, softplus, or LeakyReLU networks are universal approximators in  $L^p(\RR\times[0,1]^n)(p\ge2)$.
\subsection{A Lifting theorem}
We set some notations and recall some facts from harmonic analysis. The canonical Lebesgue measure on $\RR^m$  is denoted by $\lambda_{m}.$ The indicator function of a set $A$ is denoted by $I_A.$ We use $p_0$ and $p_1$ for the canonical projection of $\RR \times \RR^n$ to $\RR$ and $\RR^n$.
\begin{center}
\begin{tikzcd} &\RR \times \RR^n \arrow[rd, "p_1"] \arrow[ld, "p_0"']  & \\ \RR && \RR^n \end{tikzcd}
\end{center}
If $f \in L^1(\RR^n)$, then the Fourier transform of $f$ is defined by the integral form
\begin{align}\label{f1}
\widehat{f}(\xi) = \int_{x \in \RR^n} e^{-2\pi i\langle x, \xi \rangle}f(x)\mathrm{d}x,   \ \ \  \xi \in \RR^n.
\end{align}
Let $\mathscr{S}(\RR^n) \subset L^1(\RR^n)$ be the Schwartz space on $\RR^n$ (\cite{Fourier}), which is a Frechet space. Let $\mathscr{S}'$ be the space of temperate distributions, which is the dual of $\mathscr{S}$. The Fourier transform of  $u \in \mathscr{S}'$ is
\begin{align}\label{f2}
\widehat{u}(\varphi) = u(\widehat{\varphi}),   \ \ \  \varphi \in \mathscr{S},
\end{align}
where $\widehat{u} \in \mathscr{S}'$. For $p \in [1, \infty]$ there is a natural embedding $L^p(\RR^n) \to \mathscr{S}'$(\cite{Fourier}[p.135]). If $f \in L^1(\RR^n)$, then its Fourier transform defined via (\ref{f2}) agrees with its integral form (\ref{f1}). In the proof of Theorem 3.1, we need to show some function $h \in L^{q}(\RR^n)$ is zero. Because Fourier transformation is an isomorphism of $\mathscr{S}'$(\cite{Fourier}[Theorem IX.2]), it suffices to show $\widehat{h} =0.$

Our approximation results are all based on the following lifting theorem.
\begin{theorem}\label{maintheorem}
Let $n \in \NN$, $2\mathopen\leq\mathclose p\mathopen<\mathclose\infty$ and $\Omega = \RR \times [0, 1]^n$. If a real-valued function $\phi$ defined on the real line satisfies $\overline{\mathcal{S}_{1}(\phi) \cap L^{p}(\RR)} = L^{p}(\RR)$, then $\overline{\mathcal{S}_{n+1}(\phi) \cap L^{p}(\Omega)} = L^{p}(\Omega).$
\end{theorem}
\begin{proof}
Let $q = p/(p-1)$. Assume $\mathcal{S}_{n+1}(\phi) \cap L^{p}(\RR) = L^{p}(\Omega)$ is not dense in $L^p(\Omega)$. By the Hahn-Banach Theorem, there exists a nonzero $u \in  L^{p}(\Omega)^*$,  equivalently a nonzero real-valued $h \in L^q(\Omega)$, such that if $\varphi \in \mathcal{S}_{n+1}(\phi) \cap  L^{p}(\Omega)$ then
\begin{align}\label{orthogonal}
u(\varphi) = \int_{\Omega} \varphi(x) h(x) dx = 0.
\end{align}
We have $||u|| = ||h||_{L^q(\Omega)}$.  For any $\gamma \in L^{p}(\RR)$ and $y \in \RR \times \RR^n$  with $p_0(y) \neq 0$,  we claim $\gamma^{L_{y}} \in L^p(\Omega)$. By the change of variables $t_0= \langle y, x \rangle, t_i = x_i(1 \leq i \leq n)$, we have
\begin{align}\label{0607}
\left|\left|\gamma^{L_{y}}\right|\right|_{L^p(\Omega)}^p &= \int_{\RR \times [0,1]^n} \left|\gamma(\langle y, x \rangle)\right|^p dx_0dx_1\ldots dx_n  \nonumber \\
&=  |y_0|^{-1} \int_{\RR \times [0,1]^n} |\gamma(t_0)|^p dt_0dt_1\ldots dt_n \nonumber \\
&= |y_0|^{-1}||\gamma||^p_{L^p(\RR)}.
\end{align}
Hence we prove the previous claim. Moreover, by definition, for any $ \gamma \in \mathcal{S}_{1}(\phi)$,$\gamma^{L_{y}} \in \mathcal{S}_{n+1}(\phi)$. Put together, for any $\gamma \in \mathcal{S}_{1}(\phi) \cap L^{p}(\RR)$ and $y \in \RR \times \RR^n$ with $p_0(y) \neq 0$, we have $\gamma^{L_{y}} \in \mathcal{S}_{n+1}(\phi) \cap L^{p}(\Omega)$ and therefore (by (\ref{orthogonal})) 
\begin{align}\label{orthogonal0}
u(\gamma^{L_{y}}) = \int_{\Omega} \gamma^{L_{y}}(x)h(x) dx = 0.
\end{align}
We are going to prove a stronger identity: for any $\gamma \in  L^{p}(\RR)$ and $y \in \RR \times \RR^n$ with $p_0(y) \neq 0$,
\begin{align}\label{orthogonal1}
u(\gamma^{L_{y}})  = \int_{\Omega} \gamma^{L_{y}}(x)h(x) dx = 0.
\end{align}
By the assumption of our theorem, $\overline{\mathcal{S}_{1}(\phi) \cap L^{p}(\RR)} = L^{p}(\RR)$. Therefore, for any $\epsilon >0$, there exists $\gamma_{y, \epsilon} \in \mathcal{S}_{1}(\phi) \cap L^{p}(\RR)$ such that
$||\gamma-\gamma_{y, \epsilon}||_{L^p(\RR)} < ||u||^{-1} |p_0(y)|^{1/p} \epsilon$.
By (\ref{orthogonal0}), we have  $u(\gamma_{y, \epsilon}^{L_{y}})=0$. Using this identity and (\ref{0607}),
\begin{align*}
\left|u(\gamma^{L_{y}})\right| &= \left|u(\gamma^{L_{y}}-\gamma_{y, \epsilon}^{L_{y}})+ u(\gamma_{y, \epsilon}^{L_{y}})\right|  \\
&= \left|u(\gamma^{L_{y}}-\gamma_{y, \epsilon}^{L_{y}})\right| \\
&\leq ||u||\cdot||(\gamma-\gamma_{y, \epsilon})^{L_{y}}||_{L^p(\Omega)}  \\
&=   ||u|| \cdot |p_0(y)|^{-1/p} ||\gamma-\gamma_{y, \epsilon}||_{L^p(\RR)} \\
&\leq \epsilon.
\end{align*}
The above inequality holds for all $\epsilon >0$, so we have verified (\ref{orthogonal1}).
For $k \in \NN$, let $\Omega_k = [-k, k] \times [0,1]^n$. We identify $u \in L^p(\Omega)^*$ as an element of $\mathscr{S}'$ by assigning $\varphi \in \mathscr{S} \to u(\varphi \cdot I_{\Omega})$, and we define $u_k \in \mathscr{S}'$ by setting 
\begin{align*}
u_k: \varphi \in \mathscr{S} \to u(\varphi \cdot I_{\Omega_k}). 
\end{align*}
With respect to the topology of $\mathscr{S}'$ we have $\lim\limits_{k \to \infty} u_{k}=u$, and therefore
\begin{align*}
\lim_{k \to \infty} \widehat{u_{k}} = \widehat{u}.
\end{align*}
Setting $h_k = h\cdot I_{\Omega_k}$, we have $h_k \in L^1(\RR^{n+1})$. By using (\ref{orthogonal}) and the definition of $u_k$,
\begin{align}\label{20200607}
u_k(\varphi) &= u(\varphi\cdot I_{\Omega_k})        \nonumber \\
 & = \int_{\RR^{n+1}} \varphi(x) I_{\Omega_k}(x) h(x) dx \nonumber \\
 &= \int_{\RR^{n+1}}\varphi(x) h_k(x) dx.
\end{align}
Therefore, the temperate distribtuion $u_k$ is represented by an integrable function $h_k\in L^1(\RR^{n+1})$. By the fact(discussed before the proof) that in this case the Fourier transform defined via (\ref{f2}) agrees with its integral form (\ref{f1}),  $\widehat{u_{k}}$ is also represented by $\widehat{h_k}$.   The Fourier transform $\widehat{h_k}$ of $L^1$-integrable $h_k$ is represented by the integral form
\begin{align}\label{fl1}
\widehat{h_k}(\xi) = \int_{\Omega_k} e^{-2\pi i \langle x, \xi \rangle}h(x) dx.
\end{align}
Therefore, if $\varphi \in \mathscr{S}$ then
\begin{align*}
\widehat{u_k}(\varphi) &=  \int_{\RR^{n+1}}\varphi(\xi) \widehat{h_k}(\xi)d\xi  \\
&=  \int_{\RR^{n+1}}\varphi(\xi) \int_{\Omega_k} e^{-2\pi i \langle x, \xi \rangle}h(x) dx d\xi.
\end{align*}
We let $\langle A, B \rangle$ denote $\{\langle a, b \rangle : a \in A, b\in B\}$. Let $\mathcal{K}$ be any compact set in $\RR^+\times \RR^{n}$ and $\kappa=\min{ \{p_0(x):x\in\mathcal{K}\} }$.  There exists a constant $C_1$ such that $ \left| \langle \{0\} \times [0,1]^n, \mathcal{K} \rangle \right|$ is bounded from above by $C_1$. Let $K = 2C_1/\kappa$. If $k>K$, then for all $x  \in \{k\} \times [0,1]^n$ and $z  \in \mathcal{K}$, we have $p_0(x)p_0(z) >k \cdot\kappa>  (2C_1/\kappa) \cdot \kappa = 2C_1$, and $|\langle (0, p_1(x)), z \rangle| < C_1.$ Therefore $\langle x,  z  \rangle = p_0(x) p_0(z) + \langle (0, p_1(x)), z \rangle > C_1.$ By similar arguments, if $k>K$ then for all $x \in \{-k\} \times [0,1]^n$ and $z  \in \mathcal{K}$ we have  $\langle x,  z  \rangle  < -C_1$. These inequalities together with the fact $|\langle \{0\} \times [0,1]^n, \mathcal{K} \rangle |<C_1$ imply if $k >K$, then  
\begin{align}\label{1220}
\langle \{0\} \times [0,1]^n, \mathcal{K} \rangle \cap \langle \{\pm k\} \times [0,1]^n, \mathcal{K} \rangle = \emptyset.
\end{align}
For $z \in \mathcal{K}$, let $X_{z, k}^{\pm}$ denote the set of $x \in \Omega_k$ that satisfies that there exists $x^{\partial} \in \{\pm k\} \times [0, 1]^n$ satisfying $\vec{z} \perp x^{\partial}-x$:
\begin{align*}
X_{z, k}^{\pm} &= \{ x \in \Omega_{k} : \langle z,  x \rangle \in \langle z, \{ \pm k\} \times [0,1]^n \rangle \}
\end{align*}
We refer the reader to Figure \ref{fig2} for the visualization of $X_{z, k}^{\pm}$ in the two dimensional case. 
\begin{figure}[ht]\label{fig2}
\begin{center}
\centerline{\includegraphics[width=0.9\textwidth]{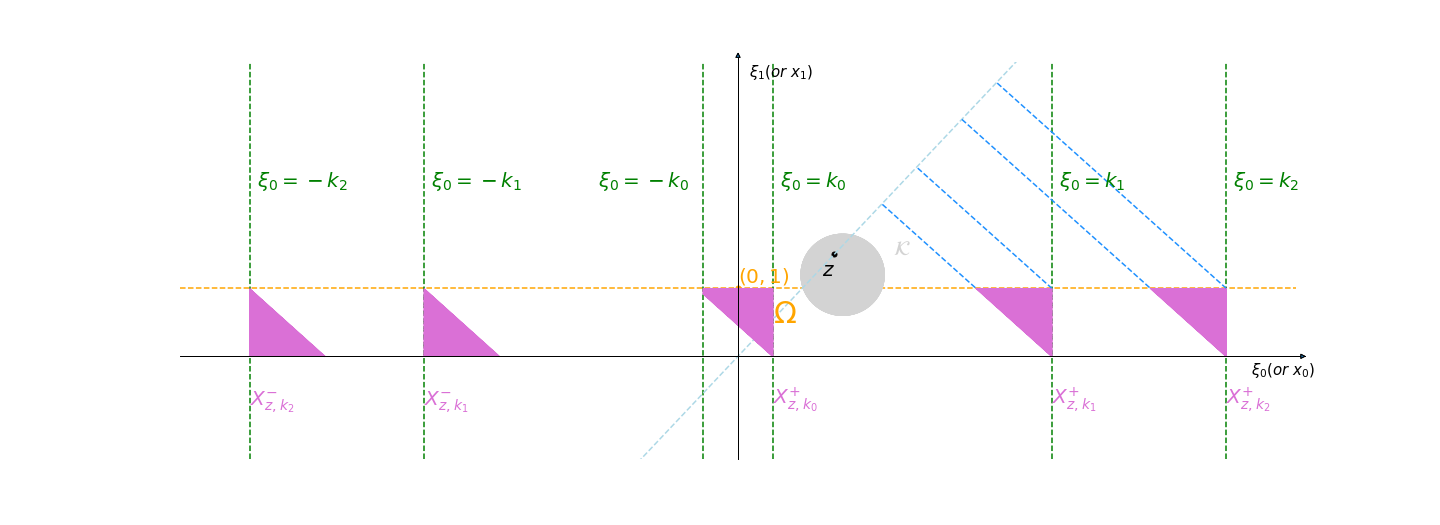}}
\caption{The contribution of $h|_{Y^{\dag}_{z, k}}$, where $Y^{\dag}_{z, k} =\Omega_k \setminus ( X_{z, k}^+\cup X_{z, k}^-) $, to $\widehat{h_k}(z
)$ is zero, therefore estimates of $\widehat{h_k}(z
)$ involve only  $X^{\pm}_{z, k}$. If $k$ is sufficiently large, $X^{\pm}_{z, k}$ become away from $\{0\} \times \RR^n$. This fact is used to prove that $X^{\pm}_{z, k}$ is stable in the sense of (\ref{1221251}).}
\label{fig2}
\end{center}
\vskip -0.2in
\end{figure}
By (\ref{1220}), if $k>K$ and $z \in \mathcal{K}$, then $\langle z, \{0\} \times [0,1]^n \rangle \cap \langle z, \{\pm k\} \times [0,1]^n \rangle = \emptyset$ and therefore $\{0\} \times [0,1]^n \cap X^{\pm}_{z, k} =  \emptyset$. On the other hand, by definition, $X_{z, k}^{+}$ and $X_{z, k}^{-}$ are connected sets, and $\{\pm k\} \times [0, 1]^n \subset X_{z,k}^{\pm}$. Therefore, for $k>K$ and $z \in \mathcal{K}$, $X_{z, k}^{\pm}$ are seperated by the axis $x_0=0$, equivalently
\begin{align}\label{1221}
X_{z, k}^{\pm} \subset \RR^{\pm} \times \RR^n.
\end{align}
Let $k_1,k_2$ be integers greater than $K$, we claim 
\begin{align}\label{1221251}
X_{z, k_2}^{\pm} = (\pm(k_2-k_1), 0,\ldots,0) + X_{z, k_1}^{\pm}. 
\end{align}
This stability property is false if $k_1,k_2$ are small. For example, in Figure \ref{fig2}, $X_{z,k_1}^+$ is bigger than $X_{z,k_0}^++(k_1-k_0,0,\ldots,0)$. For any $x \in X_{z, k_1}^{+}$, by definition, $|p_0(x)|<k_1$, and $\langle z,  x \rangle \in \langle z, \{  k_1\} \times [0,1]^n \rangle$. The following shift of $x$
\begin{align*}
x' = x+ (k_2-k_1, 0,\ldots,0)
\end{align*}
satisfies
\begin{align}\label{hz1}
p_0(x')&= p_0(x)+k_2-k_1<k_2, \nonumber \\
\langle z,  x' \rangle &\in \langle z, \{k_2\} \times [0,1]^n \rangle, \nonumber \\
p_0(x') &\ge-K>-k_2.
\end{align}
The first two follow from $|p_0(x)|<k_1$ and $\langle z,  x \rangle \in \langle z, \{  k_1\} \times [0,1]^n \rangle$. If the last inequality is not true, then $p_0(x') < -K$ and therefore $$\langle z, x' \rangle =  p_0(z) p_0(x') + \langle z, (0, p_1(x')) \rangle <-(2C_1/\kappa)\cdot\kappa+C_1=-C_1.$$
However, for $x'' \in \{k_2\} \times [0,1]^n$,  $$\langle z, x'' \rangle =  p_0(z) p_0(x'') + \langle z, (0, p_1(x'')) \rangle >(2C_1/\kappa)\cdot\kappa-C_1=C_1.$$
The above two inequalities contradict to $\langle z,  x' \rangle \in \langle z, \{k_2\} \times [0,1]^n \rangle$. Therefore, we have proved (\ref{hz1}), which implies $x' \in X_{z, k_2}^{+}$. The above argument applies to all
$x \in X_{z, k_1}^{+}$. Consequently
\begin{align*}
(k_2-k_1, 0,\ldots,0) + X_{z, k_1}^{+} \subset X_{z, k_2}^{+}.
\end{align*}
This relation remains valid if $k_1$ and $k_2$ are replaced by each other, and therefore 
\begin{align*}
(k_1-k_2, 0,\ldots,0) + X_{z, k_2}^{+} \subset X_{z, k_1}^{+}.
\end{align*}
The above two relations prove the $X_{z, \cdot}^{+}$ part of (\ref{1221251}). The $X_{z, \cdot}^{-}$ part of (\ref{1221251}) can be obtained with similar arguments. 

By (\ref{1221}) and (\ref{1221251}), for any $k>K$, the number
\begin{align*} 
\nu_{\mathcal{K}} &:= \lambda_{n+1}\left(\bigcup_{z \in \mathcal{K}} X_{z, k}^+\right)+\lambda_{n+1}\left(\bigcup_{z \in \mathcal{K}} X_{z, k}^-\right)
\end{align*}
is finite and independent of $k$.  Moreover, (\ref{1221251}) leads to
\begin{align}\label{12211503}
\bigcup_{z \in \mathcal{K}} X_{z, k_2}^{\pm} = (\pm(k_2-k_1), 0,\ldots,0) + \bigcup_{z \in \mathcal{K}} X_{z, k_1}^{\pm},
\end{align}
which implies that $\bigcup_{z \in \mathcal{K}} X_{z, k}^{\pm}$ moves to infinity as $k$ goes to infinity (see Figure \ref{fig3}).
\begin{figure}[ht]\label{fig3}
\vskip 0.2in
\begin{center}
\centerline{\includegraphics[width=0.8\textwidth]{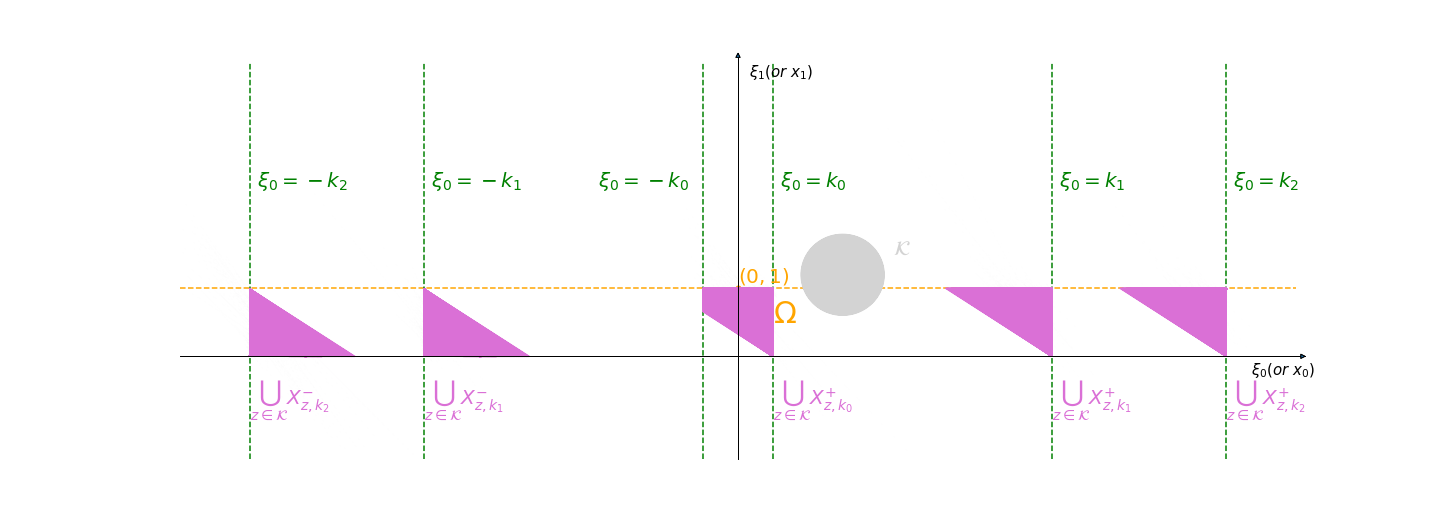}}
\caption{For the purpose of a rigorous limit argument, instead of a single $X^{\pm}_{z, k}$, our estimates rely on $\bigcup_{z \in \mathcal{K}}X^{\pm}_{z, k}$, where $\mathcal{K}$ is a compact set in $\RR^{\pm} \times \RR^{n}$.  If $k$ is sufficiently large $\bigcup_{z \in \mathcal{K}}X^{\pm}_{z, k}$ is also stable in the sense of (\ref{12211503}). }
\label{fig3}
\end{center}
\vskip -0.2in
\end{figure} 
Therefore, for any integrable function $F$, $\lim\limits_{k\to\infty}\int_{\bigcup_{z \in \mathcal{K}} X_{z, k}^{\pm}}F(x)dx =0.$

Write
\begin{align*}
X^{\dag}_{z, k} &= X_{z, k}^+\cup X_{z, k}^- \\
Y^{\dag}_{z, k} &= \Omega_k \setminus X^{\dag}_{z, k}.
\end{align*}
The following should be clear from the Figure \ref{fig2}
\begin{align}\label{12211229}
\omega \in \Omega ~ \text{and} ~ \langle z, \omega \rangle \in \langle z, Y^{\dag}_{z, k}\rangle \Leftrightarrow \omega \in Y^{\dag}_{z, k}.
\end{align}
We are going to verify (\ref{12211229}) rigorously. The $\Leftarrow$ part of (\ref{12211229}) follows from the definition. Now pick any $\omega \in \Omega$ that satisfies $\langle z, \omega\rangle \in \langle z, Y^{\dag}_{z, k}\rangle$. By assumption, there exists $w^{in} \in \Omega_k \setminus X^{\dag}_{z,k}$ such that $\langle z, \omega\rangle = \langle z, w^{in}\rangle.$ Suppose $w \notin \Omega_{k}$, then the line segment $\overline{w,w^{in}}$, which connects the point $w$ outside $\Omega_{k}$ with the point $w^{in}$ inside $\Omega_{k}$, must intersect the boundary of $\Omega_k$:
\begin{align*}
\{t w + (1-t) w^{in}: t \in [0, 1] \} \cap \{\pm k\} \times [0,1]^n \neq \emptyset.
\end{align*}
Let $w^{\partial}$ be the intersection point. We have 
\begin{align*}
\langle z, w^{\partial}\rangle = \langle z, \omega\rangle = \langle z, w^{in}\rangle. 
\end{align*}
By the definition of $X^{\dag}_{z, k}$, the above identity implies $w^{in} \in X^{\dag}_{z, k}$, which contradicts $w^{in} \in \Omega_k \setminus X^{\dag}_{z, k}$. Therefore  $w \notin \Omega_{k}$ is not ture, and instead  $w \in \Omega_{k}$. If $w$ lies in $X^{\dag}_{z, k}$, then there exists $w^{\partial} \in \{\pm k\} \times [0,1]^n$ such that 
\begin{align*}
\langle z, w^{\partial}\rangle = \langle z, \omega\rangle = \langle z, w^{in}\rangle. 
\end{align*}
Again, this identity implies $w^{in} \in X^{\dag}_{z, k}$, which contradicts $w^{in} \in \Omega_k \setminus X^{\dag}_{z, k}$. Therefore,   $w \in X^{\dag}_{z, k}$ is not ture, and instead  $w \notin X^{\dag}_{z, k}$. Putting together, $w \in \Omega_{k} \setminus X^{\dag}_{z, k} = Y^{\dag}_{z, k}$.  The above arguments prove the $\Rightarrow$ part of $(\ref{12211229})$. By (\ref{12211229}), we have
\begin{align}\label{12211429}
I_{Y^{\dag}_{z, k}} = I_{\left\{\omega: \omega \in \Omega, \langle z, \omega \rangle \in \langle z, Y^{\dag}_{z, k}\rangle \right\}}.
\end{align}
Given $z \in \mathcal{K}$, for all $k \ge K$ and $\alpha \in L^{\infty}(\RR)$, we define $\alpha_{k} = \alpha I_{\langle z, Y^{\dag}_{z,k} \rangle}$. Because $\alpha_{k}$ is an a.e. bounded function with bounded support, $\alpha_{k}  \in L^{p}(\RR)$.  Using $h_k = hI_{\Omega_{k}}$ and (\ref{12211429}), we have
\begin{align*}
& ~~~~ \int_{Y^{\dag}_{z, k}}  \alpha(\langle z, x \rangle) h_k(x) d x \\
&= \int_{\Omega}  \alpha(\langle z, x \rangle) h(x)I_{\Omega_{k}}(x)  I_{Y^{\dag}_{z, k}}(x) d x \\
 &= \int_{\Omega}  \alpha(\langle z, x \rangle) h(x)I_{Y^{\dag}_{z, k}}(x) d x \\
&= \int_{\Omega}  \alpha(\langle z, x \rangle) h(x) I_{\{\omega: \omega \in \Omega, \langle z, \omega \rangle \in \langle z, Y^{\dag}_{z, k}\rangle \}}(x) d x \\
&= \int_{\Omega}  \alpha(\langle z, x \rangle) I_{\{\omega: \langle z, \omega \rangle \in \langle z, Y^{\dag}_{z, k}\rangle \}}(x) \left(h(x) I_{\Omega}(x)\right) d x \\
&=  \int_{\Omega}  \alpha_k(\langle z, x \rangle) h(x) d x \\
& = 0.
\end{align*}
The last step of the above calculation follows from $\alpha_k \in L^{p}(\RR)$ and  (\ref{orthogonal1}).
Therefore, for any $\alpha \in L^{\infty}(\RR)$,
\begin{align*}
& ~~~~  \int_{\RR^{n+1}}  \alpha(\langle z, x \rangle) h_k(x) d x \\
&=\int_{X^{\dag}_{z, k}}  \alpha(\langle z, x \rangle) h_k(x) d x + \int_{Y^{\dag}_{z, k}}  \alpha(\langle z, x \rangle) h_k(x) d x \\
&= \int_{X^{\dag}_{z, k}}  \alpha(\langle z, x \rangle) h_k(x) d x.
\end{align*}
By this identity, the contribution of $h|_{Y^{\dag}_{z, k}}$ to $\widehat{h_k}(z)$ is zero, and estimates of $\widehat{h_k}(z)$ involve only  $X^{\pm}_{z, k}$. Let $\alpha(t) := e^{-2\pi i t}$. By (\ref{fl1}), the above identity, and  $X^{\dag}_{z, k} \subset \bigcup_{z \in \mathcal{K}} X^{\dag}_{z, k}$, we have
\begin{align}\label{1515}
\left|\widehat{h_k}(z)\right| &= \left|\int_{\RR^{n+1}}\alpha(\langle z, x \rangle) h_k(x) \mathrm{d}x\right| \nonumber \\
&= \left|\int_{X^{\dag}_{z, k}}\alpha(\langle z, x \rangle)h_k(x) \mathrm{d}x\right|  \nonumber \\
&\leq  ||\alpha^{L_z}||_{L^p(X^{\dag}_{z, k})}||h_k||_{L^q(X^{\dag}_{z, k})}  \nonumber \\
&\leq     \nu_{\mathcal{K}}^{1/p} ||h_k||_{L^q(X^{\dag}_{z, k})}
\end{align}
Let $\psi \in C^{\infty}(\RR^{n+1})$ with $\supp{\psi} \subset \mathcal{K}$, where $\mathcal{K}$ is a compact subset of  $\RR^+\times \RR^{n}$. By (\ref{1515}) and the remark after (\ref{12211503}), for any $\epsilon>0$, there exists $N_{\epsilon}$ such that if $k > N_{\epsilon}$ then 
\begin{align*}
||h||_{L^q\left(\bigcup_{z \in \mathcal{K}} X^{\dag}_{z, k} \right)} <  ||\psi||^{-1}_{L^{\infty}(\RR^{n+1})}   \lambda^{-1}_{n+1}(\mathcal{K}) \nu_{\mathcal{K}}^{-1/p}  \epsilon.
\end{align*}
By using $||h_k||_{L^q\left(\bigcup_{z \in \mathcal{K}} X^{\dag}_{z, k} \right)}\leq ||h||_{L^q\left(\bigcup_{z \in \mathcal{K}} X^{\dag}_{z, k} \right)}$ and (\ref{1515}), the above inequality leads to
\begin{align*}
\left|\widehat{h_k}(z)\right| \leq     \nu_{\mathcal{K}}^{1/p} ||h_k||_{L^q\left(\bigcup_{z \in \mathcal{K}} X^{\dag}_{z, k} \right)} <  ||\psi||^{-1}_{L^{\infty}(\RR^{n+1})}  \lambda^{-1}_{n+1}(\mathcal{K}) \epsilon .
\end{align*}
For all $k > N_{\epsilon}$, we have
\begin{align*}
\left|\widehat{u_k}(\psi)\right| &=  \left|\int_{\RR^{n+1}}\psi(z)  \widehat{h_k}(z) \mathrm{d}z\right| \nonumber \\
&=  \left|\int_{\mathcal{K}}\psi(z)  \widehat{h_k}(z) \mathrm{d}z\right| \nonumber \\
&\leq   \lambda_{n+1}(\mathcal{K}) ||\psi||_{L^{\infty}(\RR^{n+1})}  ||\psi||_{L^{\infty}(\RR^{n+1})}^{-1}   \lambda^{-1}_{n+1}(\mathcal{K})  \epsilon \\
&=  \epsilon.
\end{align*}
Consequently, for all $\psi \in \mathscr{D}(\RR^+ \times \RR^n)$, we have
\begin{align*}
\widehat{u}(\psi) = \lim_{k\to \infty}\widehat{u_k}(\psi) = 0.
\end{align*}
By similar arguments, for all $\psi \in \mathscr{D}(\RR^- \times \RR^n)$, we have
\begin{align*}
\widehat{u}(\psi) = \lim_{k\to \infty}\widehat{u_k}(\psi) = 0.
\end{align*}
The above identities imply $\widehat{u}$ is supported on the hyperplane $\{0\} \times \RR^n \subset \RR^{n+1}$. By the assumption $p \in[2, \infty)$, we have $q \in (1, 2]$, and $u$ is represented by $h\in L^q(\RR^{n+1})$. In this case $\widehat{u}$ is represented by some function in $L^p(\RR^{n+1})$ (\cite{Fourier}[Theorem IX.8]). As a $L^p$ function supported on  $\{0\} \times \RR^n \subset \RR^{n+1}$,  $\widehat{u}$ must be $0$ and therefore $u =0$. This fact contradicts $h\neq 0$, and completes the proof.
\end{proof}

\subsection{Neural networks with bounded activation functions}
We investigate approximation properties of shallow sigmoid networks in $L^p(\RR \times [0,1]^n)(2\mathopen\leq\mathclose p\mathopen <\mathclose\infty)$ by using Theorem \ref{maintheorem} and the following result \cite{SW}[Lemma 3.3]:
\begin{lemma}[Stinchcombe-White]\label{lemmasw}
If an activation function lies in $L^1(\RR) \cap L^p(\RR)$, and $\int_{\RR}\phi(t)dt \neq 0$, then $\overline{\mathcal{S}_{1}(\phi) \cap L^{p}(\RR)} = L^{p}(\RR)$.
\end{lemma}
Let $\Delta^1_{\varrho}$ be the 1st difference operator with step $\varrho$ defined by $\Delta_{\varrho}^1[f](x) = f(x+\varrho)-f(x)$, and let $\Delta^n_{\varrho} = \Delta^1_{\varrho} \circ \Delta^{n-1}_{\varrho}$ for all integers $n>1$.The following result  generalizes Lemma 1 of \cite{Funahashi}.
\begin{lemma}\label{integrable unit}
Let $\phi$ be a measurable, essentially bounded, eventually essentially monotone activation function with $\lim_{x \to \infty} \phi(x) \ a.e. \neq \lim_{x \to -\infty} \phi(x) \ a.e.$. For all $\varrho \in \RR$  and $p \in [1, \infty]$,  $\Delta_{\varrho}^1[\phi]\in L^p(\RR)$. Moreover, if $\varrho \neq 0$, then $\int_{\RR} \Delta_{\varrho}^1[\phi](x) dx \neq 0.$
\end{lemma}
\begin{proof}
Without loss of generality, we suppose 
\begin{align*}
\lim\limits_{x \to \infty}\phi(x) &= 1 \textup{ a.e.}, \\
\lim\limits_{x \to -\infty}\phi(x) &= 0 \textup{ a.e.}, \\
||\phi||_{L^{\infty}(\RR)} &= C.  
\end{align*}
The first claim of our lemma is obvious for $p = \infty$. Now take $p \in [1, \infty)$ and $\varrho\in\RR$. There exists $X>0$ such that  $|\Delta^1_{\varrho}[\phi]|$ is essentially less than 1 on $\RR \setminus (-X,X)$, and that $\phi$ is essentially monotone on $[X-\varrho, \infty)$ as well as on $(-\infty, -X+\varrho]$. The function $\Delta_{\varrho}^1[\phi]$ is essentially bounded from above by $2C$, and $\Delta_{\varrho}^1[\phi]$ is essentially non-positive (or essentially non-negative) on $(-\infty, X)$ and  $(X, \infty)$.  For $p \in [1, \infty)$, we continue with the following estimate
\begin{align*}
&~~~~ \int_{-\infty}^{+\infty} |\Delta_{\varrho}^1[\phi](x)|^p dx \\
&= \int_{-\infty}^{-X} |\Delta_{\varrho}^1[\phi](x)|^p dx+\int_{-X}^{X} |\Delta_{\varrho}^1[\phi](x)|^p dx+\int_{X}^{+\infty} |\Delta_{\varrho}^1[\phi](x)|^p dx \\
&\leq  \int_{-\infty}^{-X} |\Delta_{\varrho}^1[\phi](x)| dx+2^{p+1}XC^p+\int_{X}^{+\infty} |\Delta_{\varrho}^1[\phi](x)| dx \\
&= \left|\int_{-\infty}^{-X} \Delta_{\varrho}^1[\phi](x) dx\right|+2^{p+1}XC^p+\left|\int_{X}^{+\infty} \Delta_{\varrho}^1[\phi](x) dx\right| \\
&=  \lim_{M \to \infty} \left|\int_{-M}^{-X} \Delta_{\varrho}^1[\phi](x) dx\right| + \lim_{M \to \infty} \left|\int_{X}^{M}  \Delta_{\varrho}^1[\phi](x) dx\right| + 2^{p+1}XC^p \\
&=  \lim_{M \to \infty} \left|\int_{-X}^{-X+\varrho}\phi(x) dx - \int_{-M}^{-M+\varrho}\phi(x) dx \right|+\\
&~~~~~ \lim_{M \to \infty} \left|\int_{M}^{M+\varrho}\phi(x) dx - \int_{X}^{X+\varrho}\phi(x) dx \right| + 2^{p+1}XC^p \\
& =  \left|\int_{-X}^{-X+\varrho} \phi(x) dx \right|+\left|\varrho - \int_{X}^{X+\varrho} \phi(x) dx \right| + 2^{p+1}XC^p  \\
&< \infty.
\end{align*}
Therefore for all $\varrho \in \RR$ and $p \in [1, \infty)$, $\Delta_{\varrho}^1[\phi] \in L^p(\RR)$. Moreover,
\begin{align*}
&~~~~ \int_{-\infty}^{+\infty} \Delta_{\varrho}^1[\phi](x) dx  \\ 
&=  \lim_{Y \to \infty} \int_{-Y}^{Y} (\phi(x+\varrho) - \phi(x)) dx \\
&= \lim_{Y \to \infty} \left(\int_{Y}^{Y+\varrho} \phi(x) dx - \int_{-Y}^{-Y+\varrho} \phi(x) dx\right) \\
&= \varrho.
\end{align*}
This identity completes the proof of our lemma.
\end{proof}

For any $\phi$ that satisfies the requirements of Lemma \ref{integrable unit}, there  exists $\varrho$ such that $\Delta_{\varrho}^1[\phi] \in L^1(\RR) \cap L^p(\RR)$ and $\int_{-\infty}^{+\infty}\Delta_{\varrho}^1[\phi](x) dx \neq 0.$ By Lemma \ref{lemmasw},  $\mathcal{S}_{1}(\Delta_{\varrho}^1[\phi])$ is dense in $L^{p}(\RR)$. Because $\mathcal{S}_{1}(\Delta_{\varrho}^1[\phi]) \subset \mathcal{S}_{1}(\phi)$, $\overline{\mathcal{S}_{1}(\phi) \cap L^{p}(\RR)} = L^{p}(\RR).$ Therefore, we have proved
\begin{lemma}\label{sigmoid}
Let $\phi$ be a measurable, essentially bounded, eventually essentially monotone activation function with $\lim \limits_{x \to \infty} \phi(x) \ a.e. \neq \lim\limits_{x \to -\infty} \phi(x) \ a.e.$. For $1\mathopen\leq\mathclose p\mathopen\leq\mathclose\infty$, $\overline{\mathcal{S}_{1}(\phi) \cap L^{p}(\RR)} = L^{p}(\RR).$
\end{lemma}
Lemma \ref{sigmoid} and Theorem \ref{maintheorem} implies the following theorem, which applies to sigmoids.
\begin{theorem}\label{theorembounded11}
Let $\Omega = \RR \times [0, 1]^n$, $\phi$ be a measurable, essentially bounded, eventually essentially monotone activation function with $\lim \limits_{x \to \infty} \phi(x) \ a.e. \neq \lim\limits_{x \to -\infty} \phi(x) \ a.e.$. For $2\mathopen\leq\mathclose p\mathopen<\mathclose\infty$,
\begin{align*}
\overline{\mathcal{S}_{n+1}(\phi) \cap L^{p}(\Omega)} = L^{p}(\Omega).
\end{align*}
\end{theorem}

\subsection{Neural networks with unbounded activation functions}
To investigate the approximation capabilities of a shallow ReLU, ELU, Softplus, or LeakyReLU network, we start with the following lemma.
\begin{lemma}
Let $n \in \NN$, and $\phi$ be an eventually $(n\mathopen+\mathclose1)$-times differentiable activation function. If $\phi^{(n)}$ is eventually monotone, then $\Delta_1^n[\phi]$ is also eventually monotone.
\end{lemma}
\begin{proof}
Because $\phi$ is eventually $(n\mathopen+\mathclose1)$-times differentiable, there exists $X_1$ such that if $|x|\mathopen>\mathclose X_1$ then 
\begin{align}\label{difference}
(\Delta_1^{n}[\phi])'(x) = \Delta_1^{n}[\phi'](x) = (\phi')^{(n)}(\overline{x})= (\phi^{(n)})'(\overline{x})
\end{align}
for some $\overline{x} \in [x, x+n].$ Because $\phi^{(n)}$ is eventually monotone and eventually differentiable, there exists $X_2$ such that $\phi^{(n+1)}|_{(X_2, \infty)}$ and $\phi^{(n+1)}|_{(-\infty, -X_2)}$ are non-positive or non-negative. Set $X = \max{\{X_1, X_2+n\}}$. By (\ref{difference}),  $(\Delta_1^{n}[\phi])'|_{(X, \infty)}$ and $(\Delta_1^{n}[\phi])'|_{(-\infty, -X)}$ are non-positive or non-negative. This implies that $\Delta_1^n[\phi]$ is eventually monotone.
\end{proof}
With the above lemma, we prove
\begin{lemma}\label{relu}
Let $1\mathopen\leq\mathclose p\mathopen<\mathclose\infty$, and $\phi$ be \textup{ReLU},  \textup{ELU},  \textup{Softplus}, or \textup{LeakyReLU}.  We have
\begin{align*}
\overline{\mathcal{S}_{1}(\phi) \cap L^{p}(\RR)} = L^{p}(\RR).
\end{align*}
\end{lemma}
\begin{proof}
If $\phi$ is \textup{ReLU},  \textup{ELU},  \textup{Softplus}, or \textup{LeakyReLU}, then $\phi'$ is eventually monotone. By the previous lemma, $\Delta_1^1[\phi]$ is also eventually monotone. Moreover, we have
\begin{align*}
\Delta_1^1[\textup{ReLU}](t) & \to \left\{
\begin{array}{rcl}
1 & ~ as&{t \to +\infty,}\\
0 & ~ as&{t\to -\infty};
\end{array}\right. \\
\Delta_1^1[\textup{ELU}](t) & \to \left\{
\begin{array}{rcl}
1 & ~ as&{t \to +\infty,}\\
0 & ~ as&{t\to -\infty};
\end{array}\right. \\
\Delta_1^1[\textup{Softplus}](t) & \to \left\{
\begin{array}{rcl}
1 & ~ as&{t \to +\infty,}\\
0 & ~ as&{t\to -\infty};
\end{array}\right. \\
\Delta_1^1[\textup{LeakyReLU}](t) & \to \left\{
\begin{array}{rcl}
1 & ~ as&{t \to +\infty,}\\
\alpha & ~ as&{t \to -\infty};
\end{array}\right. 
\end{align*}
Therefore, $\Delta_1^1[\phi]$ is a bounded, eventually monotone function with different limits at $\pm \infty$. By Lemma \ref{sigmoid}, $\overline{\mathcal{S}_{1}(\Delta_1^1[\phi]) \mathopen\cap\mathclose L^p(\RR)} \mathopen=\mathclose L^p(\RR)$. Because $\mathcal{S}_{1}(\Delta_1^1[\phi]) \mathopen\subset\mathclose \mathcal{S}_{1}(\phi)$, we have $\overline{\mathcal{S}_{1}(\phi) \mathopen\cap\mathclose L^p(\RR)} \mathopen=\mathclose L^p(\RR)$.
\end{proof}
By the above lemma and Theorem \ref{maintheorem}, we have
\begin{theorem}\label{theorembounded}
Let $2\mathopen\leq\mathclose p\mathopen<\mathclose\infty, \Omega \mathopen=\mathclose \RR \mathopen\times\mathclose [0, 1]^n$, and $\phi$ be \textup{ReLU},  \textup{ELU},  \textup{Softplus}, or \textup{LeakyReLU}. We have
\begin{align*}
\overline{\mathcal{S}_{n+1}(\phi) \cap L^{p}(\Omega)} = L^{p}(\Omega).
\end{align*}
\end{theorem}
Theorem \ref{theorembounded11} and  Theorem \ref{theorembounded} remain valid if $\RR \times [0, 1]^n$ is replaced by $\RR \times K$, where $K$ is any bounded set in $\RR^n$.

\section{Inexpressivity}
In this section, we prove that a shallow sigmoid, ReLU, ELU, Softplus, or LeakyReLU network expresses no nonzero function in $L^p(\RR\times\RR^{+}) (1\mathopen\leq\mathclose p\mathopen<\mathclose\infty)$.
\subsection{The inexpressivity of sums of ridge functions}
A real-valued function $F$ is called a ridge function, if there exist $\phi: \RR \mathopen\to\mathclose \RR$ and $y \mathopen\in\mathclose \RR^n$ such that $F \mathopen=\mathclose \phi^{L_{y}}$. By definition, if $F$ is a ridge function and $\alpha$ is a real number, then $\alpha F$ is also a ridge function. Because $\phi^{\tau_{\varrho} L_{y}}$ is a ridge function,  any function represented by a shallow neural network is a finite sum of ridge functions. The space $L^2(\RR^2)$ contains no nonzero ridge function \cite{ridge}[Proposition 1.1], which implies a nonzero ReLU unit is not integrable. To investigate the inexpressivity of neural networks, we need a stronger result on the finite sum of ridge functions.
\begin{proposition}\label{prop}
Let $n \mathopen\in\mathclose \NN, \Omega \mathopen=\mathclose \RR \mathopen\times\mathclose \RR^+$, $y_k \mathopen\in\mathclose \RR^2(1\mathopen\leq\mathclose k\mathopen\leq\mathclose n)$, $F_k(1\mathopen\leq\mathclose k\mathopen\leq\mathclose n)$ real valued functions  defined on $\RR$, and $F  \mathopen= \sum_{k=1}^n F_k^{L_{y_k}}$. Assume that
\begin{enumerate}
\item For $1\mathopen\leq\mathclose k\mathopen\leq\mathclose n$ and $\omega \mathopen\in\mathclose \{\mathopen\pm\mathclose\infty\}$, there exist $\beta_{k, \omega},\alpha_{k, \omega}$ such that $\lim\limits_{t \to \omega}(F_{k}(t) \mathopen-\mathclose \beta_{k, \omega}t) \mathopen=\mathclose \alpha_{k, \omega},$
\item $F|_{\Omega}$ is uniformly continuous,
\item $F|_{\Omega} \in L^p(\Omega)$ for some $1\mathopen\leq\mathclose p\mathopen<\mathclose\infty$.
\end{enumerate}
Then we have $F\mathopen=\mathclose0.$
\end{proposition}
\begin{proof}
Throughout the proof, we let $re^{i\theta}$ denote the point $(r\cos{\theta}, r\sin{\theta})\in\RR^2$. If the proposition is not valid, then there exists a nonzero sum of ridge functions $F=\sum_{k=1}^nf_k^{L_{y_k}}$ such that all three assumptions in the proposition are satisfied. With the following relationship among ridge functions
\begin{align*}
G^{L_{\rho e^{i \theta}}} + H^{L_{s e^{i \theta}}} &=  (G^{L_{\rho}}+H^{L_{s}})^{L_{e^{i\theta}}}, \\
G^{L_{e^{i(\theta+\pi)}}}  &= (G^{L_{-1}})^{L_{e^{i\theta}}},
\end{align*}
we can reorganize $f_k$, via transforming those with $\theta>\pi$ and summing up those with the same $\theta$, so that the sum is in the form of
\begin{align}\label{F}
F =\sum\limits_{k=1}^n F_k^{L_{e^{i\theta_k}}},
\end{align}
where $\theta_k$ are different numbers in $[0, \pi)$. The first assumption of the proposition remains valid after the reorganization: for all $1\leq k \leq n$ and $\omega \in \{\pm\infty\}$, there exist real numbers $\beta_{k, \omega},\alpha_{k, \omega}$ such that 
\begin{align*}
\lim_{t \to \omega}(F_{k}(t) - \beta_{k, \omega}t) = \alpha_{k, \omega}.
\end{align*}
The last two assumptions of the proposition also remain valid after the reorganization. We claim:
\begin{align}\label{claim}
F_k^{L_{e^{i\theta_k}}}|_{\Omega} \textup{ are all linear functions.} 
\end{align}
If this is not true, then there exists $j \in \{1,\ldots, n\}$ such that $F_j^{L_{e^{i\theta_j}}}$ is not a linear function. By nonlinearity, there exists $o=(o_1, o_2) \in \Omega$ such that 
\begin{align}\label{aaaa}
F_j^{L_{e^{i\theta_j}}}(o) &\neq \alpha_{j, \infty}+ \beta_{j, \infty}(o_1\cos{\theta_j}+o_2\sin{\theta_j}). 
\end{align}
The function $\overline{F}: x \in \RR^2 \mapsto F(x+o)$ has the following properties:
\begin{enumerate}
\item[$P1$.] $\overline{F} \in L^p(\Omega)$ and $\overline{F}|_{\Omega}$ is uniformly continuous.
\item[$P2$.] There exist functions $\overline{F_k}$ such that $\overline{F} =\sum_{k=1}^n \overline{F_k}^{L_{e^{i\theta_k}}}$. For all $k$ and $\omega \in \{\pm\infty\}$, there are real numbers $\overline{\alpha_{k, \omega}}$ satisfying
\begin{align*}
\lim_{t \to \omega}(\overline{F_{k}}(t) - {\beta_{k, \omega}}t) = \overline{\alpha_{k, \omega}}.
\end{align*}
\item[$P3$.] $\overline{F_j}(0) \neq \overline{\alpha_{j, \infty}}$.
\end{enumerate}
\noindent The property $P1$ follows from $||\overline{F}||_{L^p(\Omega)} \leq ||F||_{L^p(\Omega)}$. The property $P2$ follows from
\begin{align*}
\overline{F}(x) & = \sum\limits_{k=1}^n F_k( (x_1 + o_1) \cos\theta_k+ (x_2 + o_2) \sin\theta_k) \\
 & = \sum\limits_{k=1}^n F_k( x_1\cos\theta_k+ x_2\sin\theta_k +(o_1\cos\theta_k+o_2\sin\theta_k)  ) \\
& = \sum_{k=1}^n \overline{F_k}^{L_{e^{i\theta_k}}}(x),
\end{align*}
where $\overline{F_k}(x) = F_k(x+o_1\cos{\theta_k}+o_2\sin{\theta_k})$, and from
\begin{align*}
&~~~~\lim\limits_{t \to \infty}(\overline{F_k}(t)-\beta_{k, \infty}t) \\
&= \lim\limits_{t \to \infty}(F_k(t+o_1\cos{\theta_k}+o_2\sin{\theta_k}) - \beta_{k, \infty}t) \\
&= \lim\limits_{t \to \infty}(F_k(t) - \beta_{k, \infty}t) + \beta_{k, \infty}(o_1\cos{\theta_k}+o_2\sin{\theta_k}) \\
& = \alpha_{k, \infty}+ \beta_{k, \infty}(o_1\cos{\theta_k}+o_2\sin{\theta_k}).
\end{align*}
The property $P3$ follows from (\ref{aaaa}) and the following facts:
\begin{align*}
\overline{\alpha_{j, \infty}}&=  \alpha_{j, \infty}+ \beta_{j, \infty}(o_1\cos{\theta_j}+o_2\sin{\theta_j}), \\
\overline{F_j}(0) &=F_j^{L_{e^{i\theta_j}}}(o). 
\end{align*}
For any $\theta \in [0,\pi)$, by the property $P2$, the end behavior of $\overline{F}(re^{i\theta})$ at $\infty$ is similar to a linear function in $r$. Therefore, $\lim\limits_{r \to \infty}\overline{F}(re^{i\theta})$ exists in $\RR \cup \{\pm\infty\}$. If $\lim\limits_{r \to \infty}\overline{F}(re^{i\theta}) \neq 0$, then
there exists $\rho > 0$ and $C>0$ such that if $r>\rho$ then $|\overline{F}(re^{i\theta})|\ge C.$ Because $\overline{F}$ is uniformly continuous, there exists $\delta >0$ such that if $|x-y|<\delta$ then $|\overline{F}(x)-\overline{F}(y)|<C/2.$ Let $\eta = \min{\{\rho (\pi-\theta), \delta\}}$. If $r > \rho$ and $\vartheta \in (\theta, \theta+\eta/r)$, then $re^{i\vartheta} \in \Omega$ and
\begin{align*}
|re^{i\vartheta}-re^{i\theta}| =  2r\sin{(\vartheta/2 - \theta/2)}< r(\vartheta - \theta) < \eta \leq \delta.
\end{align*}
Therefore, $|\overline{F}(re^{i\vartheta}) - \overline{F}(re^{i\theta})|<C/2$, and consequently
\begin{align*}
|\overline{F}(re^{i\vartheta})| &\ge |\overline{F}(re^{i\theta})|-|\overline{F}(re^{i\vartheta}) - \overline{F}(re^{i\theta})| \\
& \ge C-C/2 \ge C/2.
\end{align*}
By the above inequality, we have
\begin{align*}
\int_{\Omega} \left|\overline{F}(x_1, x_2)\right|^p \mathrm{d}x_1\mathrm{d}x_2 &\ge \int_{\rho}^{\infty} r \int_{\theta}^{\theta+\eta/r}\left|\overline{F}(re^{i\vartheta})\right|^p \mathrm{d}\vartheta\mathrm{d}r \\
&\ge \int_{\rho}^{\infty} \eta (C/2)^p \mathrm{d}r,
\end{align*}
which contradicts  $\overline{F} \in L^p(\Omega)$. Therefore, $\lim\limits_{r \to \infty}\overline{F}(re^{i\theta}) \neq 0$ is not true. Instead, for all $\theta \in [0, \pi)$,
\begin{align}\label{is0}
\lim\limits_{r \to \infty}\overline{F}(re^{i\theta})=0.
\end{align}
We notice
\begin{align*}
\overline{F_k}^{L_{e^{i\theta_k}}}(re^{i\theta}) &=\overline{F_k}(r\cos{\theta}\cos{\theta_k}+r\sin{\theta}\sin{\theta_k}) \\
&=\overline{F_k}(r\cos{(\theta-\theta_k)}).
\end{align*}
This identity together with $\overline{F} =\sum_{k=1}^n \overline{F_k}^{L_{e^{i\theta_k}}}$ and (\ref{is0}) implies for all $\theta \in [0, \pi)$, 
\begin{align}\label{is01}
\lim\limits_{r \to \infty}\sum_{k=1}^n \overline{F_k}(r\cos{(\theta-\theta_k)})=0.
\end{align}
For $\theta \in [0, \pi)$, let $\theta^{\perp}$ denote the unique number in $[0, \pi)$ such that $\cos{(\theta - \theta^{\perp})} =0$,  equivalently $\theta^{\perp}=\{ \theta \pm \pi/2 \} \cap [0, \pi)$. There exists a small $\omega_{\theta_j} >0$ such that $(\theta_j^{\perp}, \theta_j^{\perp}+\omega_{\theta_j}) \subset (0, \pi)$ and $\{\theta_k^{\perp} | 1 \leq k \leq n\} \cap (\theta_j^{\perp}, \theta_j^{\perp}+\omega_{\theta_j}) = \emptyset$. For every $\theta \in (\theta_j^{\perp}, \theta_j^{\perp}+\omega_{\theta_j})$ and $k \in \{1, \ldots, n\}$, we have
\begin{align}\label{singularinfinity}
\cos(\theta-\theta_k) \neq 0. 
\end{align}
Otherwise $\theta = \{\theta_k \pm \pi/2 \} \cap [0, \pi) = \theta_k^{\perp}$, which contradicts $\theta_k^{\perp} \notin  (\theta_j^{\perp}, \theta_j^{\perp}+\omega_{\theta_j}).$ If $\cos(\theta-\theta_k) > 0$, then by the property $P2$, we have
\begin{align*}
\lim\limits_{r \to \infty} (\overline{F_k}(r\cos{(\theta-\theta_k)})-\beta_{k, \infty}\cos{(\theta-\theta_k)}r) = \overline{\alpha_{k, \infty}}.
\end{align*} 
If $\cos(\theta-\theta_k) < 0$, then by the property $P2$, we have
\begin{align*}
\lim\limits_{r \to \infty} (\overline{F_k}(r\cos{(\theta-\theta_k)})-\beta_{k, -\infty}\cos{(\theta-\theta_k)}r) = \overline{\alpha_{k, -\infty}}.
\end{align*} 
If $\cos(\theta-\theta_k) = 0$, then
\begin{align*}
\lim\limits_{r \to \infty} \overline{F_k}(r\cos{(\theta-\theta_k)}) = \overline{F_k}(0).
\end{align*} 
Let $\theta$ be a number in $(\theta_j^{\perp}, \theta_j^{\perp}+\omega_{\theta_j})$. By (\ref{singularinfinity}), for all $k$, there exists $\omega_k \in \{\pm\infty\}$ such that
\begin{align}\label{06081}
\lim\limits_{r \to \infty} (\overline{F_k}(r\cos{(\theta-\theta_k)})-\beta_{k, \omega_k}\cos{(\theta-\theta_k)}r) = \overline{\alpha_{k, \omega_k}}.
\end{align}
This identity together with (\ref{is01}) leads to $\sum_{k=1}^n\beta_{k, \omega_k}\cos{(\theta-\theta_k)} = 0$ and
\begin{align}\label{sorry1}
\overline{\alpha_{j, \omega_j}} = -\sum_{k\neq j}\overline{\alpha_{k, \omega_k}}.
\end{align}
At $\theta=\theta_j^{\perp}$, for $k \neq j$, we have $\cos{(\theta_j^{\perp}-\theta_k)} \neq 0$ and therefore
\begin{align}\label{06082}
\lim\limits_{r \to \infty} (\overline{F_k}(r\cos{(\theta_j^{\perp}-\theta_k)})-\beta_{k, \omega_k}\cos{(\theta_j^{\perp}-\theta_k)}r) &= \overline{\alpha_{k, \omega_k}}.
\end{align}
Here $\omega_k$ in (\ref{06082}) is the same as the $\omega_k$ in (\ref{06081}), as $\cos{(\theta_j^{\perp}-\theta_k)}$ is of the same sign as $\cos{(\theta-\theta_k)}$, where $\theta\in(\theta_j^{\perp}, \theta_j^{\perp}+\omega_{\theta_j})$.
For the particular $\overline{F_j}$, we have
\begin{align*}
\lim\limits_{r \to \infty} \overline{F_j}(r\cos{(\theta_j^{\perp}-\theta_j)}) &=  \overline{F_j}(0).
\end{align*}
The above two identities together with (\ref{is01}) lead to $\sum_{k\neq j}\beta_{k, \omega_k}\cos{(\theta_j^{\perp}-\theta_k)} =0$ and
\begin{align}\label{sorry2}
\overline{F_j}(0) = -\sum_{k\neq j}\overline{\alpha_{k, \omega_k}}.
\end{align}
By (\ref{sorry1}) and (\ref{sorry2}), we have $\overline{F_j}(0) = \overline{\alpha_{j, \infty}}$, which contradicts the property $P3$. Therefore, we have proved (\ref{claim}), which claims that $F_k^{L_{e^{i\theta_k}}}|_{\Omega}$ are all linear functions. So is $F|_{\Omega}$. As a linear function in $L^p(\Omega)$, $F$ must be zero.
\end{proof}

The domain $\RR \times \RR^+$ is optimal for Proposition \ref{prop}. For example,
\begin{example}
Let $c$ be any positive number, and $\Omega=\{(t, x) \in \RR^2: |x|<ct \}.$ The sum of ridge functions, $\textup{ReLU}(t+2)-2\textup{ReLU}(t+1)+\textup{ReLU}(t)$, lies in $L^1(\Omega)$.
\end{example}

\subsection{The inexpressivity of neural networks}
By using Proposition \ref{prop}, we can prove
\begin{proposition}\label{theorembounded111}
Let $\Omega \mathopen=\mathclose \RR \mathopen\times\mathclose \RR^+$, and $\phi$ be a real-valued, measurable, essentially bounded activation function with $\lim\limits_{x \to \infty} \phi(x) \ a.e.$ and $\lim\limits_{x \to -\infty} \phi(x) \ a.e.$ existing. For  $1\mathopen\leq\mathclose p\mathopen<\mathclose\infty$,
\begin{align*}
\mathcal{S}_{2}(\phi) \cap L^{p}(\Omega)= \{0\}.
\end{align*}
\end{proposition}
\begin{proof} 
Suppose there exists a positive integer $n$, real numbers $\{t_i\}_{i=1}^n, \{\varrho_i\}_{i=1}^n$ and nonzero vectors  $\{y_i \in \RR^2\}_{i=1}^n$ such that the function 
\begin{align*}
\Phi =t_0+ \sum_{i=1}^n t_i\phi^{\tau_{\varrho_i}L_{y_i}}
\end{align*}
lies in $L^p(\Omega)$.  Take a smooth function $\rho$ that satisfies $\int_{\RR^2}\rho(x)\mathrm{d}x = 1$ and $\supp{\rho} \subset \{x \in \RR \times \RR^- : |x|<1 \}.$ For $\epsilon>0$, define $\rho_{\epsilon}(x) = \epsilon^{-2}\rho(x/\epsilon)$. For all $z = (z_1, z_2) \in \Omega$, 
\begin{align}\label{singleconv}
\phi^{\tau_{\varrho_i}L_{y_i}} * \rho_{\epsilon} (z) &= \int_{\RR^2} \phi^{\tau_{\varrho_i}L_{y_i}}(z-x) \rho_{\epsilon}(x)\mathrm{d}x_1 \mathrm{d}x_2 \nonumber \\
 &= \int_{\RR^2} \phi(\langle y_i, z \rangle -\langle y_i, x\rangle + \varrho_i ) \rho_{\epsilon}(x)\mathrm{d}x_1  \mathrm{d}x_2 \nonumber \\
 &= \phi_{\epsilon, i}^{L_{y_i}}(z),
\end{align}
where $\phi_{\epsilon, i}$ are given by
\begin{align*}
\phi_{\epsilon, i}(s) =  \int_{\RR^2} \phi(s -\langle y_i, x\rangle + \varrho_i ) \rho_{\epsilon}(x)\mathrm{d}x_1  \mathrm{d}x_2.
\end{align*}
Summing up (\ref{singleconv}) for all $i$, we have
\begin{align}\label{globalconv}
\Phi * \rho_{\epsilon} = t_0 + \sum_{i=1}^n t_i\phi_{\epsilon, i}^{L_{y_i}}.
\end{align}
We claim $\phi_{\epsilon, i}$ are continuous: if $\lim\limits_{j\to \infty}s_j =s$ then $\lim\limits_{j\to \infty}\phi_{\epsilon, i}(s_j) =\phi_{\epsilon, i}(s).$ Since $|y_i| \neq 0$, we can let 
\begin{align*}
z_j = \frac{s-s_j}{|y_i|^2}y_i.
\end{align*}
Then $\langle y_i, z_j \rangle =s-s_j$ and $\lim\limits_{j\to \infty}z_j =(0,0).$ Using $\phi$ is essentially bounded, $\rho_{\epsilon}$ is bounded and has bounded support, and the dominated convergence theorem, we have
\begin{align*}
\lim_{j\to \infty}\phi_{\epsilon, i}(s_j) &= \lim_{j\to \infty}\int_{\RR^2} \phi(s_j -\langle y_i, x\rangle + \varrho_i ) \rho_{\epsilon}(x)\mathrm{d}x \\
&=\lim_{j\to \infty}\int_{\RR^2} \phi(s -\langle y_i, x+z_j\rangle + \varrho_i ) \rho_{\epsilon}(x)\mathrm{d}x \\
&=\lim_{j\to \infty}\int_{\RR^2} \phi(s -\langle y_i, x\rangle + \varrho_i ) \rho_{\epsilon}(x-z_j)\mathrm{d}x \\
&=\int_{\RR^2} \phi(s -\langle y_i, x\rangle + \varrho_i ) \rho_{\epsilon}(x)\mathrm{d}x \\
&=\phi_{\epsilon, i}(s).
\end{align*}
This calculation verifies our claim that $\phi_{\epsilon, i}$ are continuous. Moreover, it is trivial to check that $\lim_{s \to \pm \infty}\phi_{\epsilon, i}(s) = \lim_{s \to \pm \infty} \phi(s) \ a.e.$. As a continuous function with bounded limits at the infinities, $\phi_{\epsilon, i}$ must be uniformly continuous. As a composition of $\phi_{\epsilon, i}$ with a linear function (which is uniformly continuous), $\phi_{\epsilon, i}^{L_{y_i}}$ is also uniformly continuous. Consequently, $\Phi * \rho_{\epsilon}$ is uniformly continuous.  In particular, $\Phi * \rho_{\epsilon}|_{\Omega}$ (\ref{globalconv}) is uniformly continuous. 

Next, we claim $\Phi * \rho_{\epsilon}|_{\Omega} \in L^p(\Omega)$. By assumption, we only have  $\Phi|_{\Omega} \in L^p(\Omega)$ but not $\Phi \in L^p(\RR^2)$, so the argument will be slightly lengthier than an expected one. Let $\overline{\Phi} = \Phi I_{\Omega}$, then $\overline{\Phi} \in L^p(\RR^2)$.   By the facts that $\rho_{\epsilon}|_{\Omega}=0$, and that  if $z \in \RR\times\RR^+, x \in \RR \times \RR^-$ then $z-x \in \RR\times\RR^+$, for $z \in \RR\times\RR^+$, we have
\begin{align*}
\overline{\Phi} * \rho_{\epsilon} (z) &= \int_{\RR \times \RR} \overline{\Phi}(z-x) \rho_{\epsilon}(x)\mathrm{d}x_1 \mathrm{d}x_2 \\
&= \int_{\RR \times \RR^-} \overline{\Phi}(z-x) \rho_{\epsilon}(x)\mathrm{d}x_1 \mathrm{d}x_2 \ \ \ \ (\rho_{\epsilon}|_{\RR\times\RR^+}=0) \\
&= \int_{\RR \times \RR^-} \Phi(z-x) \rho_{\epsilon}(x)\mathrm{d}x_1 \mathrm{d}x_2 \ \ \ \ (z-x \in \RR\times\RR^+)\\
&= \int_{\RR \times \RR} \Phi(z-x) \rho_{\epsilon}(x)\mathrm{d}x_1 \mathrm{d}x_2  \ \ \ \ (\rho_{\epsilon}|_{\RR\times\RR^+}=0)\\
&= \Phi * \rho_{\epsilon} (z).
\end{align*}
Therefore, $\overline{\Phi} * \rho_{\epsilon}|_{\Omega} = \Phi * \rho_{\epsilon}|_{\Omega}$.
By \cite{measure}[Theorem 3.9.4], $\overline{\Phi} * \rho_{\epsilon} \in L^p(\RR^2)$, hence $\overline{\Phi} * \rho_{\epsilon}|_{\Omega}$ and $\Phi * \rho_{\epsilon}|_{\Omega}$ (\ref{globalconv}) are $L^p$-integrable. 

The first assumption of  Proposition 4.1 is also satisfied by $\Phi * \rho_{\epsilon}|_{\Omega}$ (\ref{globalconv}), as $\phi_{\epsilon, i}$ have finite limits at $\{\pm \infty\}$. We can now apply Proposition 4.1 to (\ref{globalconv}) and obtain $ \Phi * \rho_{\epsilon}|_{\Omega} =0$. 

This identity leads to $\overline{\Phi} * \rho_{\epsilon}|_{\Omega}  =0$. As $\epsilon$ is arbitrary chosen this is true for all $\epsilon>0$. By \cite{measure}[Theorem 4.24] and the fact that $\overline{\Phi} \in L^p(\RR^2)$,  $\lim\limits_{\epsilon \to 0} \overline{\Phi} * \rho_{\epsilon} = \overline{\Phi}$ in $L^p(\RR^2)$ and in particular $\lim\limits_{\epsilon \to 0} \overline{\Phi} * \rho_{\epsilon}|_{\Omega} = \overline{\Phi}|_{\Omega}$ in $L^p(\Omega)$. This identity together with $\overline{\Phi} * \rho_{\epsilon}|_{\Omega}  =0$ leads to $\overline{\Phi}|_{\Omega} = 0$. By the definition of $\overline{\Phi}$,  $\overline{\Phi}|_{\Omega}  = \Phi|_{\Omega} $. Therefore, $\Phi|_{\Omega} = 0$.
\end{proof}

We prove the following simple lemma
\begin{lemma}
If  $\phi$ is continuous and eventually uniformly continuous, then $\phi$ is uniformly continuous.
\end{lemma}
\begin{proof}
There exists  a positive number $X$ such that on intervels $J_1=(-\infty, -X], J_2=[-X-1, X+1], J_3=[X, \infty)$,  $\phi|_{J_j}(j = 1, 2, 3)$ are all uniformly continuous. Suppose the lemma is not true. There exists some positive $\epsilon$ and a pair of real numbers $\{x_i, y_i\}$, for all $i\in\NN$, such that $|x_i-y_i|<1/i$ and $|f(x_i)-f(y_i)|>\epsilon$. Each pair $\{x_i, y_i\}$ is contained in at least one of $J_j$, hence there exists $J \in \{J_1, J_2, J_3\}$ and a subsequence $\{i_k\}_{k \in \NN}$ of $\{i\}_{i \in \NN}$ such that $\{x_{i_k}, y_{i_k}\} \subset J$ for all $k$. This contradicts the fact that $\phi$ is uniformly continuous on $J$. 
\end{proof}

As being  eventually Lipschitz implies being eventually uniformly continuous, the above lemma implies
\begin{lemma}\label{Lipschitz}
If $\phi$ is continuous and eventually Lipschitz, then it is uniformly continuous.
\end{lemma}
With this lemma, we prove
\begin{proposition}\label{theorembounded2}
Let $\Omega \mathopen=\mathclose \RR \mathopen\times\mathclose \RR^+$, and $\phi$ be \textup{ReLU}, \textup{ELU},\textup{Softplus}, or \textup{LeakyReLU}. For $1\mathopen\leq\mathclose p\mathopen<\mathclose\infty$,
\begin{align*}
\mathcal{S}_{2}(\phi) \cap L^{p}(\Omega)=\{0\}.
\end{align*}
\end{proposition}
\begin{proof} 
Take $F \mathopen\in\mathclose \mathcal{S}_{2}(\phi)$. There exist a positive integer $k$,  real numbers $t_i,\varrho_i$ and vectors  $y_i \mathopen\in\mathclose \RR^2$ such that  $F=\sum_{i=1}^{k}(t_i\phi^{\tau_{\varrho_i}})^{L_{y_i}}.$ Because
\begin{align*}
\lim_{x\to \infty}(\textup{ReLU}(x)\mathopen-\mathclose x) &\mathopen=\mathclose 0, & \lim_{x\to \infty}(\textup{ELU}(x)\mathopen-\mathclose x) &\mathopen=\mathclose 0,   \\
\lim_{x\to -\infty}\textup{ReLU}(x) &\mathopen=\mathclose 0, & \lim_{x\to -\infty}\textup{ELU}(x) &\mathopen=\mathclose -\alpha, \\
 \lim_{x\to \infty}(\textup{Softplus}(x)\mathopen-\mathclose x) &\mathopen=\mathclose 0, & \lim_{x\to \infty}(\textup{LeakyReLU}(x)\mathopen-\mathclose x) &\mathopen=\mathclose 0, \\
 \lim_{x\to -\infty}\textup{Softplus}(x) &\mathopen=\mathclose 0, & \lim_{x\to -\infty}(\textup{LeakyReLU}(x)\mathopen-\mathclose \alpha x) &\mathopen=\mathclose 0,  
\end{align*}
the first assumption of Proposition \ref{prop} is satisfied for $F$. By checking derivatives around $\pm \infty$, ReLU, ELU, Softplus, and LeakyReLU are all eventually Lipschitz. By Lemma \ref{Lipschitz}, $\phi$ is uniformly continuous. As compositions of uniformly continuous functions,  $(t_i\phi^{\tau_{\varrho_i}})^{L_{y_i}}(x)=t_i\phi(\langle y_i, x\rangle+\varrho_i)$ are also uniformly continuous. Therefore $F$ is uniformly continuous, which satisfies the second assumption of Proposition \ref{prop}. For any $F\mathopen\in\mathclose L^{p}(\Omega)$, we can apply Proposition \ref{prop} and conclude that $F\mathopen=\mathclose 0$.
\end{proof}

Proposition \ref{theorembounded111} and Proposition \ref{theorembounded2} prove the inexpressivity of neural networks in $L^{p}(\RR\times\RR^+)$. Because $\RR\times\RR^+$ is a submanifold of many unbounded domains, by the Fubini theorem, our results lead to the inexpressivity of neural networks on many other spaces. For example, we have
\begin{corollary}\label{corollarydepth}
Let $n\mathopen\ge\mathclose 1$ and $\phi$ be a sigmoid, \textup{ReLU}, \textup{ELU},\textup{Softplus}, or \textup{LeakyReLU}. For $1\mathopen\leq\mathclose p\mathopen<\mathclose\infty$,
\begin{align*}
\mathcal{S}_{n+2}(\phi) \cap L^{p}(\RR^2 \times [0,1]^n)=\{0\}, \\
\mathcal{S}_{n+1}(\phi) \cap L^{p}(\RR^{n+1})=\{0\}, \\
\mathcal{S}_{n+1}(\phi) \cap L^{p}(\{(x_0, \ldots, x_n)| x_1^2\mathopen+\cdots+\mathclose x_n^2\mathopen>\mathclose x_0\})=\{0\}.
\end{align*}
\end{corollary}
\subsection{Benefit of Depth} 
The advantage of multiple hidden layer model was discussed in \cite{Pinkus}. There has also been growing interest in investigating the benefit of depth \cite{depth}, \cite{Eldan2016ThePO}. Our Corollary \ref{corollarydepth} and results of \cite{deeprelu}, \cite{width} provide a new perspective on this subject for  ReLU networks.
\begin{example}
Let $n\mathopen \ge\mathclose 2$, and $1\mathopen\leq\mathclose p\mathopen<\mathclose  \infty$. The shallow \textup{ReLU} network expresses no nonzero function in $L^p(\RR^n)$, while the deep \textup{ReLU} network provides universal approximation in $L^p(\RR^n)$.
\end{example}


\section{Deep ReLU networks}
Theorem 2.3 of \cite{deeprelu} tells us that deep ReLU networks with at most $\lceil\log_2(n+1)\rceil$ hidden layers can arbitrarily well approximate any function in $L^p(\RR^n)$. In this section, we improve this result by proving that deep ReLU networks with 2 hidden layers is already a universal approximator in $L^p(\RR^n)$. Given a function $F$ defined on $\RR^n$, let $S_F$ consist of all functions $\sum_{i=1}^qt_iF\left((x-z_i)/\sigma\right), q\in \NN, t_i\mathopen\in\mathclose\RR, \sigma\mathopen>\mathclose0, z_i \mathopen\in\mathclose\RR^n.$  We will use
\begin{theorem}[\cite{Park}]\label{thmpark}
Let $F:\RR^n \mathopen\to\mathclose \RR$ be integrable, bounded, and a.e. continuous with $\int_{\RR^n}F \mathopen\neq\mathclose 0$. Then $S_F$ is dense in $L^p(\RR^n)$ for $p\mathopen\in\mathclose [1, \infty).$
\end{theorem}
\noindent and the following simple lemmas
\begin{lemma}
If $F$ and $G$ are functions defined on $\RR^n$ and are both represented by the ReLU network with depth $d_h$, then $F+G$ is also represented by the ReLU network with depth $d_h$.
\end{lemma}
\begin{lemma}
If $F$ is a function defined on $\RR^n$ and is represented by the ReLU network with depth $d_h$, and $\Lambda:\RR^n \to \RR^n$ is an affine map, then $F\circ \Lambda$ is also represented by the ReLU network with depth $d_h$. 
\end{lemma}
We now prove
\begin{theorem}
The ReLU network with depth 3(equivalent to 2 hidden layers) is a universal approximator in $L^p(\RR^n)$.
\end{theorem}
\begin{proof}
Let $\rho=ReLU$, and let $G$ be a non-negative-valued function on $\RR$ defined by, for $x\in\RR$,
$$G(x) =\rho(x)-2\rho(x-1)+\rho(x-2).$$
Then $G$ satisfies $G(x)=0$ for $x\notin [0,2]$ and $0\leq 0 \leq 1$ for $x\in [0,2]$. We construct a function $F$ on $\RR^n$ defined by, for $x=(x_1, \ldots, x_n)$,
\begin{align*}
F(x) = G(G(x_1)+\cdots+G(x_n)-(n-1)).
\end{align*}
Then for $x\in \RR^n\setminus [0,2]^n$, we have $G(x_1)+\cdots+G(x_n)-(n-1)\leq n-1-(n-1)=0$ and therefore $F(x)=0$. This function $F$ satisfies:  $F$ is continuous and non-negative; $F$ is of bounded support as  for $x\in \RR^n\setminus [0,2]^n$, $F(x)=0$; $F$ is not a vanishing function as $F((1,\ldots,1))=1$. By using Theorem \ref{thmpark}, $S_F$ is dense in $L^p(\RR^n)$. Moreover, $F$ is represented by ReLU network with depth 3, and therefore by the above simple lemmas, all functions in $S_F$ are  represented by ReLU network with depth 3. Consequently, ReLU network with depth 3 is a universal approximator in $L^p(\RR^n)$. The general case $d_h \ge 3$ follows from similar arguments.
\end{proof}

\bibliographystyle{plain}
\bibliography{acnn}

\end{document}